\documentclass[letterpaper, 10 pt, conference]{ieeeconf}  
\usepackage[utf8]{inputenc}
\usepackage[T1]{fontenc}
\usepackage{amsmath}
\usepackage{amssymb}
\usepackage{algorithm}
\usepackage{algpseudocode}
\newtheorem{assumption}{Assumption}
\let\labelindent\relax
\usepackage{enumitem}
\usepackage{cuted}
\newtheorem{lemma}{Lemma}
\newtheorem{theorem}{Theorem}
\usepackage{graphicx}

\IEEEoverridecommandlockouts                              

\title{\LARGE \bf
BiCQL-ML: A Bi-Level Conservative Q-Learning Framework for Maximum Likelihood Inverse Reinforcement Learning
}

\author{Junsung Park$^{1}$%
\\$^{1}$Department of Electrical and Computer Engineering, Seoul National University}

\begin{document}

\maketitle
\thispagestyle{empty}
\pagestyle{empty}

\begin{abstract}

Offline inverse reinforcement learning (IRL) aims to recover the underlying reward function that explains expert behavior using only static demonstration data, without access to online environment interaction. We propose BiCQL-ML, a novel policy-free offline IRL algorithm that jointly optimizes a reward function and a conservative Q-function in a bi-level framework, thereby entirely avoiding explicit policy learning. Specifically, our method alternates between (i) learning a conservative Q-function via Conservative Q-Learning (CQL) under the current reward, and (ii) updating the reward parameters to maximize the expected Q-values of expert actions while mitigating over-generalization to out-of-distribution (OOD) actions. This procedure can be interpreted as a maximum likelihood estimation (MLE) framework under soft value matching. We establish theoretical guarantees that BiCQL-ML converges to a reward function under which the expert policy is soft-optimal. Empirically, our approach consistently outperforms competitive baselines (BC, DAC, ValueDice) across MuJoCo control tasks and pre-collected datasets from the D4RL benchmark suite, demonstrating both the effectiveness and robustness of the proposed offline IRL framework.

\end{abstract}

\section{Introduction}

Designing appropriate reward functions is one of the most fundamental and persistent challenges in reinforcement learning (RL), especially in complex real-world domains where the desired behaviors are difficult to specify analytically. Inverse Reinforcement Learning (IRL)~\cite{Abbeel2004,Fu2018} provides a principled framework to infer the underlying reward function from expert demonstrations, thereby bypassing the need for manual reward engineering.

Recently, offline IRL~\cite{Rath2021,ValueDice} has gained traction due to its practical importance in real-world scenarios, where interaction with the environment is often costly, risky, or infeasible. Offline IRL aims to recover reward functions solely from static datasets collected a priori, without requiring online sampling. However, learning reliable rewards in this setting introduces several significant challenges:

\textbf{Out-of-Distribution (OOD) Generalization:} Offline datasets only cover a subset of the full state-action space. Existing IRL algorithms may extrapolate poorly outside this distribution, often overestimating rewards for unseen actions, resulting in misaligned or unsafe learned behaviors.

\textbf{Reward Ambiguity:} Without additional structure or regularization, multiple reward functions may explain the same expert behavior, leading to ill-posed inference and non-unique solutions.

\textbf{Policy-Dependence and Computational Cost:} Many IRL methods rely on explicitly learning or optimizing a policy during training. For instance, MaxEnt IRL~\cite{Ziebart2008} and AIRL~\cite{Fu2018} perform repeated policy optimization, which is expensive and unstable in high-dimensional offline settings.

While recent offline IRL methods have attempted to address these challenges, they each come with notable limitations. AIRL~\cite{Fu2018} and GAIL~\cite{Ho2016} fundamentally rely on interactive rollouts and adversarial training, rendering them incompatible with purely offline RL settings. These methods require on-policy sampling from the environment during training, which violates the core assumption of offline learning where no additional data collection is permitted. ValueDICE~\cite{ValueDice} reformulates IRL as a density ratio estimation problem, but still depends on policy modeling and suffers in the presence of severe distributional shift, as density ratio estimation becomes unreliable for out-of-distribution state-action pairs. Offline ML-IRL~\cite{Zeng2023} presents a maximum-likelihood framework for reward inference, yet its use of generative world models and policy evaluation steps results in increased structural complexity and computational overhead. These additional components make the method less stable and harder to scale to high-dimensional tasks compared to more lightweight alternatives.

Our key insight is that explicit policy optimization is not necessary for accurate reward inference in the offline setting. Instead, by focusing solely on conservative value estimation, we can infer reward functions that are both robust and generalizable, even in the presence of severe out-of-distribution (OOD) shifts. Crucially, our approach operates entirely within the offline regime: it never requires interaction with the environment or rollout-based policy updates.

\vspace{0.5em}
In this paper, we propose a novel offline IRL framework, BiCQL-ML, that entirely bypasses policy estimation and instead alternates between two key steps:
\begin{enumerate}
    \item \textbf{Conservative Q-function estimation:} We adopt Conservative Q-Learning (CQL)~\cite{Kumar2020} to learn a value function that penalizes overestimation in OOD regions using the current reward function. This reduces the risk of reward misalignment due to value extrapolation.
    \item \textbf{Reward function update via Maximum Likelihood:} We then update the reward parameters to maximize the likelihood of expert actions under the current Q-values using a soft Boltzmann policy. This ensures that expert actions are assigned higher values than alternatives.
\end{enumerate}

\vspace{0.5em}
\textbf{Our contributions are summarized as follows:}
\begin{itemize}
    \item We propose the offline IRL framework based on conservative value learning and maximum likelihood reward inference.
    \item Our algorithm eliminates the need for adversarial training, explicit policy optimization, or world models—improving simplicity, stability, and scalability.
    \item We provide theoretical guarantees for convergence and expert optimality supported by rigorous analysis.
    \item We empirically demonstrate that our method outperforms prior IRL methods (BC, DAC, ValueDICE) on MuJoCo-based D4RL benchmarks in both reward alignment and robustness.
\end{itemize}

\section{Preliminaries and Problem formulation}

\subsection{Conservative Q-Learning}

Conservative Q-Learning (CQL)~\cite{Kumar2020} is an offline reinforcement learning (RL) algorithm that addresses the overestimation of Q-values for out-of-distribution (OOD) actions. In standard Q-learning, value functions may assign high values to actions not present in the dataset, leading to poor policy performance in offline settings. CQL mitigates this issue by incorporating a conservative penalty into the objective, explicitly pushing down Q-values of unseen or rarely observed actions.

Formally, CQL augments the standard Bellman error minimization objective with a regularization term that penalizes high Q-values for actions sampled from a broader distribution $\mu(a|s)$ (e.g., uniform or policy-induced), while maintaining high values for actions in the dataset:

\begin{align}
\min_Q \ &\underbrace{\mathbb{E}_{(s,a)\sim \mathcal{D}}\left[\left(Q(s,a) - \mathcal{B} Q(s,a)\right)^2\right]}_{\text{Bellman Error}} \notag \\
&+ \alpha \cdot \underbrace{\left(\mathbb{E}_{s\sim \mathcal{D}, a\sim \mu}\left[Q(s,a)\right] - \mathbb{E}_{(s,a)\sim \mathcal{D}}\left[Q(s,a)\right]\right)}_{\text{Conservative Penalty}}
\end{align}

where $\mathcal{B}Q(s,a)$ denotes the Bellman backup and $\alpha > 0$ controls the strength of the penalty. This encourages learning a conservative Q-function that underestimates values for unobserved actions, promoting safer policy improvement in the offline regime.

\subsection{Maximum Likelihood-based Inverse Reinforcement Learning}

Offline Maximum Likelihood Inverse Reinforcement Learning (Offline ML-IRL)~\cite{Zeng2023} formulates reward learning as a bi-level maximum likelihood estimation (MLE) problem. 

To mitigate the challenges posed by distributional shift and model uncertainty inherent in offline settings, Offline ML-IRL first constructs a conservative dynamics model $\hat{P}(s'|s,a)$ from static transition data. To penalize unreliable or out-of-distribution (OOD) state-action pairs, an uncertainty-based regularizer $U(s,a)$ is introduced. The learned policy is also encouraged to maintain high entropy through an entropy regularization term $\mathcal{H}(\pi(\cdot|s))$, promoting robustness and stochastic exploration.

This approach leads to a bi-level optimization framework:
\begin{align}
    \max_\theta \; \widehat{L}(\theta), \quad \text{s.t.} \quad \pi_\theta := \arg\max_\pi \; \mathbb{E}_{\tau \sim (\eta, \pi, \hat{P})}\\
    \Bigg[ \sum_{t=0}^{\infty} \gamma^t \big( 
    r(s_t, a_t; \theta) + U(s_t, a_t) \notag
    & + \mathcal{H}(\pi(\cdot|s_t)) \big) \Bigg]
\end{align}

Where, $\widehat{L}(\theta)$ denotes the surrogate log-likelihood of expert trajectories under the induced policy $\pi_\theta$, and $r(s,a;\theta)$ represents the reward function parameterized by $\theta$. The term $\hat{P}$ refers to the transition model learned from offline data, while $U(s,a)$ serves as a penalty function that discourages the selection of uncertain or poorly supported state-action pairs. The entropy regularizer $\mathcal{H}(\pi(\cdot|s))$ is included to promote stochasticity and robustness in the learned policy. The algorithm alternates between updating the reward parameters $\theta$ to maximize expert likelihood, and optimizing the policy $\pi_\theta$ via conservative RL under the current reward.

\subsection{Problem Formulation}
Our objective is to learn a reward function \( r_\theta(s, a) \) such that the induced policy \( \pi_\theta(a \mid s) \) assigns the highest likelihood to the expert's actions \( a_t \) given the observed states \( s_t \). Formally, we aim to maximize the likelihood of the expert trajectories under the learned policy:

\[
\max_\theta \; \mathbb{E}_{(s_t, a_t) \sim \mathcal{D}_E} \left[ \log \pi_\theta(a_t \mid s_t) \right]
\]

We formalize the offline inverse reinforcement learning (IRL) objective as a bi-level optimization problem, where the upper-level optimizes a reward function that maximizes the likelihood of expert behavior, and the lower-level solves a conservative Q-learning problem conditioned on that reward:

\begin{equation}
\max_{\theta} \ \mathcal{L}(\theta) \quad \text{subject to} \quad Q_\phi = \arg\min_{Q} \ \mathcal{L}_{\text{CQL}}(Q; \theta, \mathcal{D}).
\end{equation}

Here, $\mathcal{L}(\theta)$ denotes the expected log-likelihood of expert demonstrations under a soft policy $\pi_Q(a|s) \propto \exp(Q_\phi(s,a))$, and $\mathcal{L}_{\text{CQL}}$ represents the conservative Q-learning loss that estimates the Q-function under the current reward $r_\theta(s,a)$. The optimization is performed over the reward parameters $\theta$ using expert demonstrations $\mathcal{D}_E = \{(s_i, a_i)\}_{i=1}^{N}$, while the Q-function is learned using offline transition data $\mathcal{D} = \{(s, a, s')\}$.

Our method alternates between reward learning and conservative Q-function learning. This bi-level structure allows the reward to be inferred without explicit policy learning, while ensuring that value estimates remain reliable even without environment interaction. An overview of the full optimization framework is illustrated in Figure 1. The final outputs are an optimized reward function that explains the expert demonstrations, and a conservative Q-function that can be used for policy evaluation or control in downstream tasks.

\begin{figure}
    \centering
    \includegraphics[scale=0.45]{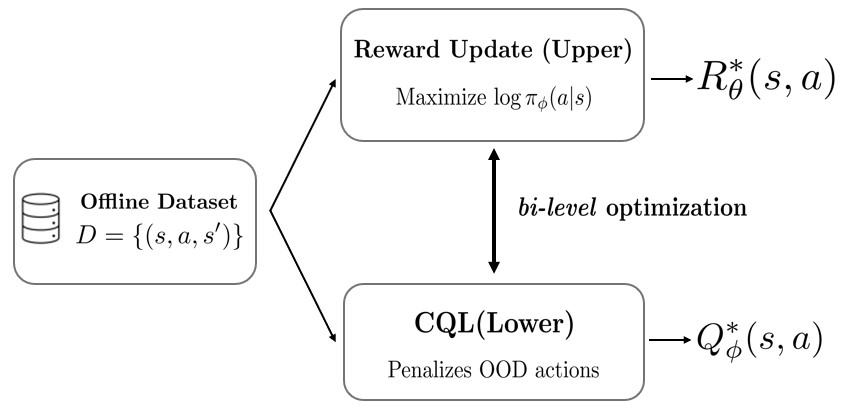}
    \caption{Bi-level offline IRL algorithm: the lower level uses Conservative Q-Learning (CQL) to learn a conservative Q-function $Q^*_\phi(s,a)$, while the upper level updates the reward $R^*_\theta(s,a)$ to maximize expert likelihood under the induced Boltzmann policy. Both components are trained iteratively using only offline data.}
\end{figure}

\section{Proposed Method}
In this section, we propose a new offline Inverse Reinforcement Learning (IRL) algorithm based on the Maximum Likelihood framework. Instead of directly learning a policy, our method alternates between optimizing a Q-function $Q_\phi(s,a)$ and a reward function $R_\theta(s,a)$, both represented by neural networks. The objective is to (i) maximize the likelihood of expert demonstrations and (ii) conservatively evaluate the Q-values using static offline data.

\subsection{Upper-Level Optimization – Reward Learning via Expert Likelihood}

In the upper-level optimization stage, we update the reward function $r_\theta(s,a)$ to increase the likelihood of expert actions under the soft Boltzmann policy induced by the Q-function $Q_\phi(s,a)$. 

We assume the expert follows a soft policy of the Boltzmann form:
\begin{equation}
    \pi_\phi(a|s) = \frac{\exp(Q_\phi(s,a))}{\sum_{a'} \exp(Q_\phi(s,a'))}
\end{equation}
Given an expert state-action pair $(s_e, a_e) \sim D_E$, the log-likelihood of the expert action under this policy is:
\begin{equation}
    \log \pi_\phi(a_e|s_e) = Q_\phi(s_e, a_e) - \log \sum_{a} \exp(Q_\phi(s_e, a))
    \label{eq:loglik}
\end{equation}

To make expert behavior more likely, we maximize the expected log-likelihood over expert demonstrations, resulting in the following reward objective:
\begin{equation}
    L_R(\theta) = \mathbb{E}_{(s_e,a_e) \sim D_E} \left[ Q_\phi(s_e, a_e) - \log \sum_{a} \exp(Q_\phi(s_e, a)) \right]
    \label{eq:reward_objective}
\end{equation}

While this objective reflects the standard maximum likelihood estimation (MLE) of expert behavior under a soft Boltzmann policy, it becomes non-informative for reward learning when the Q-function \( Q_\phi \) is fixed. Specifically, since \( Q_\phi \) is pre-trained and held constant during reward optimization, the gradient of the log-likelihood with respect to the reward parameters \( \theta \) vanishes:

\begin{equation}
\nabla_\theta \mathcal{L}_R(\theta) = 0.
\end{equation}

To address this, we reinterpret the MLE objective as a surrogate signal and adopt an indirect reward learning approach. In particular, we observe that the quantity
\[
Q_\phi(s,a) - \gamma \cdot \log \sum_{a'} \exp(Q_\phi(s', a'))
\]
serves as a soft advantage function under the current Q-function. Motivated by this, we formulate reward learning as a regression problem that aligns the reward function \( r_\theta(s,a) \) with this soft advantage. Specifically, we minimize the following squared loss:

\begin{align}
\mathcal{L}_r(\theta) = \; & \mathbb{E}_{(s,a,s') \sim \mathcal{D}} \Bigg[ \Bigg( r_\theta(s,a) \nonumber \\
& - \left( Q_\phi(s,a) - \gamma \cdot \log \sum_{a'} \exp(Q_\phi(s',a')) \right) \Bigg)^2 \Bigg]
\label{eq:reward_regression}
\end{align}

where \( \mathcal{D} \) denotes the offline dataset of transitions. This formulation enables gradient-based optimization of \( \theta \), allowing the reward function to approximate the soft advantage values implied by the current Q-function. We update \( \theta \) via stochastic gradient descent:
\begin{equation}
\theta \leftarrow \theta - \eta_r \cdot \nabla_\theta \mathcal{L}_r(\theta),
\end{equation}
where \( \eta_r \) denotes the learning rate for the reward network.

 Although the reward is not directly optimized via the MLE gradient, this surrogate regression target preserves the original objective’s structure by emphasizing higher rewards for expert-preferred actions and lower rewards for alternatives. As such, it provides a practical and stable mechanism for reward learning in the presence of a fixed Q-function.

\subsection{Lower-Level Optimization – Conservative Q-Function Learning}

At the lower level of our IRL framework, the Q-function $Q_\phi$ is updated to conservatively evaluate the current reward function $R_\theta$ using offline data. We adopt Conservative Q-Learning (CQL)~\cite{Kumar2020}, extending it to a maximum-entropy (soft-RL) setting. This modification replaces the standard Bellman operator with a soft Bellman operator, incorporating an entropy term to encourage diverse and robust policy behaviors.

The Q-function parameters $\phi$ are optimized by minimizing the following total loss:
\begin{equation}
    \mathcal{L}_Q(\phi) = \mathcal{L}_{\mathrm{BE}}(\phi) + \mathcal{L}_{\mathrm{CQL}}(\phi)
    \label{eq:Q_loss_total}
\end{equation}
where $\mathcal{L}_{\mathrm{BE}}$ is the soft Bellman error term and $\mathcal{L}_{\mathrm{CQL}}$ is the conservative regularizer.

The soft Bellman error is defined as:
\begin{equation}
    \mathcal{L}_{\mathrm{BE}}(\phi) = \frac{1}{2} \mathbb{E}_{(s,a,s') \sim D} \left[ Q_\phi(s,a) - \left( R_\theta(s,a) + \gamma V_{\phi^-}(s') \right) \right]^2
    \label{eq:bellman_error}
\end{equation}

where $V_{\phi^-}(s')$ is the soft value function defined as:
\begin{equation}
V_{\phi^-}(s') = \log \sum_{a'} \exp(Q_{\phi^-}(s', a'))
\end{equation}
and $Q_{\phi^-}$ represents a target Q-network whose parameters are periodically synchronized with $Q_\phi$ to stabilize training.

To mitigate overestimation and ensure robustness to out-of-distribution (OOD) actions, we incorporate the CQL regularization term:
\begin{align}
    \mathcal{L}_{\mathrm{CQL}}(\phi) 
    &= \alpha( \mathbb{E}_{s \sim D} \biggl[
        \log \sum_{a} \exp(Q_\phi(s,a))\biggr]   \nonumber \\
    &\quad - \mathbb{E}_{a \sim D(s)} \biggl[Q_\phi(s,a)\biggr])
    \label{eq:cql_reg}
\end{align}

Here, $\alpha > 0$ is a hyperparameter controlling the degree of conservatism. The first term serves as an upper bound to discourage high Q-values globally, while the second term encourages high Q-values specifically for actions present in the dataset. Minimizing this term helps ensure that the learned Q-function suppresses inflated estimates for unseen or less frequent actions.

The optimization of $\mathcal{L}_Q(\phi)$ is conducted via stochastic gradient descent (SGD). At each update step, a mini-batch is sampled from the offline dataset $D$, and gradients $\nabla_\phi \mathcal{L}_Q(\phi)$ guide the Q-function parameter updates. The Q-function parameters are updated via gradient descent:
\begin{equation}
    \phi \leftarrow \phi + \eta_Q \nabla_\phi L_Q(\phi)
\end{equation}
where $\eta_Q$ is the learning rate for the Q-function update. We alternate Q-learning updates for multiple gradient steps until it accurately reflects the current reward function.

\subsection{Overall Algorithm}
The overall procedure follows a bi-level optimization between conservative Q-function learning and reward maximization steps. The full algorithm is provided in Appendix~\ref{alg:bi_level_irl}.

\section{Theoretical Analysis and Convergence}
In this section, we provide a theoretical analysis of the proposed algorithm. We establish its convergence properties and demonstrate that the learned reward function ensures the expert's policy is soft-optimal.

\subsection{Convergence of Bi-Level Optimization}
We begin by stating the key assumptions necessary for our analysis.

\begin{assumption}[Lipschitz Continuity \& Boundedness]\label{assump:lipschitz}
\end{assumption}
\begin{enumerate}[label=(\alph*)]
    \item Bounded domains and rewards.  
    The state space $\mathcal{S}$ and action space $\mathcal{A}$ are finite or compact, and the reward function is uniformly bounded:
    \[
        |r(s,a)| \leq R_{\max}, \quad \forall (s,a) \in \mathcal{S} \times \mathcal{A}.
    \]
    This condition ensures that the optimal Q-values are also bounded.

    \item Lipschitz continuity of the reward function.
    The reward function $r_\theta(s,a)$ is $L_r$-Lipschitz continuous with respect to the parameter $\theta$. As a result, the optimal Q-function $Q^*(\theta)$ obtained by solving the lower-level MDP with reward $r_\theta$ is also Lipschitz continuous in $\theta$. That is, there exists a constant $L_Q < \infty$ such that:
    \[
        \| Q^*(\theta_1) - Q^*(\theta_2) \|_\infty \leq L_Q \| \theta_1 - \theta_2 \|, \quad \forall \theta_1, \theta_2.
    \]
    For example, if $r_\theta$ is linear in $\theta$ and the discount factor satisfies $\gamma \in (0,1)$, then we have:
    \[
        L_Q \le \frac{L_r}{1 - \gamma}.
    \]

    \item Lipschitz reward parameter inference.
    The reward parameter update mapping $\theta(Q) := \arg\max_\theta \mathcal{L}(\theta; Q)$ is $L_{\text{ML}}$-Lipschitz continuous with respect to $Q$, i.e.,
    \[
        \| \theta(Q) - \theta(Q') \| \leq L_{\text{ML}} \| Q - Q' \|_\infty.
    \]

    \item Contraction condition. 
    The product of the two Lipschitz constants satisfies:
    \[
        L_Q \cdot L_{\text{ML}} < 1,
    \]
    which ensures that the bi-level update forms a contraction mapping and thus converges.
\end{enumerate}

\begin{lemma}[Contraction of Composite Update]\label{lemma:contraction}
Define the composite update $G: \Theta \to \Theta$ by
\[
G(\theta) := F_{\text{upper}}(F_{\text{lower}}(\theta)),
\]
where $F_{\text{lower}}(\theta)$ returns the optimal $Q$-function for reward $r_\theta$, and $F_{\text{upper}}(Q)$ returns the maximizing reward parameters given $Q$. Under Assumption~\ref{assump:lipschitz}, $G$ is a contraction mapping on the reward parameter space $\Theta$. In particular, for any $\theta, \theta' \in \Theta$,
\[
\|G(\theta) - G(\theta')\| \leq L_{\text{ML}} L_Q \| \theta - \theta' \|,
\]
and since $L_{\text{ML}} L_Q < 1$ by assumption, $G$ has a unique fixed point and the iteration $\theta_{k+1} = G(\theta_k)$ converges to it.
\end{lemma}

\begin{theorem}[Convergence to Fixed Point]\label{thm:convergence}
Under Assumption~\ref{assump:lipschitz}, the bi-level optimization procedure converges to a unique fixed point $(\theta^*, Q^*)$. Furthermore, the pair $(\theta^*, Q^*)$ satisfies the bi-level optimality conditions:
\begin{align*}
    &Q^* = Q^*(\theta^*) \quad \text{(lower-level optimality)}, \\
    &\theta^* = \arg\max_{\theta} \mathcal{L}(\theta; Q^*) \quad \text{(upper-level optimality)},
\end{align*}
where $Q^*(\theta^*)$ denotes the conservative $Q$-function solving the lower-level optimization (e.g., Conservative Q-Learning) for reward $r_{\theta^*}$, and $\mathcal{L}(\theta; Q)$ denotes the expert log-likelihood objective. Hence, $(\theta^*, Q^*)$ is the unique bi-level fixed point of the alternating update procedure.
\end{theorem}

\subsection{Expert Policy Optimality}

We now formalize the conditions under which the expert policy is optimal under the learned reward.

\begin{assumption}[Expert Policy and Identifiability]\label{assump:expert}
\end{assumption}
\begin{enumerate}[label=(\alph*)]
    \item Expert as Soft-Optimal Policy.  
    The expert demonstrations are generated by a stochastic policy $\pi_E(a \mid s)$, which we assume follows a soft Boltzmann form:
    \[
        \pi_E(a \mid s) \propto \exp\big(Q_E(s,a)\big),
    \]
    for some unknown expert $Q$-function $Q_E(s,a)$. Equivalently,
    \[
        \pi_E(a \mid s) = \frac{\exp(Q_E(s,a))}{\sum_{a' \in \mathcal{A}} \exp(Q_E(s,a'))}
    \]
    We assume $\pi_E$ has full support over its domain (i.e., any $(s,a)$ pair not visited is outside the expert’s support).

    \item Identifiability of the Reward Function.
    The reward function is parameterized (e.g., linearly) as $r_\theta$, and the expert likelihood objective
    \[
        \mathcal{L}(\theta) = \mathbb{E}_{(s,a) \sim \mathcal{D}_E} \big[ \log \pi_\theta(a \mid s) \big]
    \]
    is concave in $\theta$. Here, $\pi_\theta(a \mid s) \propto \exp(Q_\theta(s,a))$ is the soft-optimal policy induced by reward $r_\theta$. Furthermore, we assume the expert demonstrations are informative enough to uniquely determine $\theta^* = \arg\max_\theta \mathcal{L}(\theta)$ (up to symmetry).
\end{enumerate}

\begin{lemma}[Likelihood Optimality Condition]\label{lemma:likelihood_opt}
\end{lemma}

Let $r^*(s,a) = r_{\theta^*}(s,a)$ be the learned reward function. Then for any state $s$ in the support of $\pi_E$, the expert’s action $a_E$ satisfies:
\[
    Q^*(s,a_E) \ge Q^*(s,a), \quad \forall a \in \mathcal{A},
\]
where $Q^*$ is the optimal $Q$-function under $r^*$. That is, the expert action maximizes $Q^*(s,\cdot)$ at each visited state.

\begin{theorem}[Expert Optimality under Learned Reward]\label{thm:expert_optimal}
Under the learned reward function $r^*(s,a)$, the expert policy $\pi_E$ achieves the highest expected return among all policies. That is, $\pi_E$ is an optimal policy for the MDP $\langle \mathcal{S}, \mathcal{A}, P, r^*, \gamma \rangle$, and satisfies:
\[
    J(\pi_E) \ge J(\pi), \quad \forall \pi,
\]
where $J(\pi)$ is the expected discounted return. Equivalently, $V^*(s)$ is achieved by $\pi_E$ at all $s$ visited in the expert data.
\end{theorem}

\begin{proof}
From Lemma~\ref{lemma:likelihood_opt}, $\pi_E$ is greedy with respect to $Q^*$ on the support of the expert data. Consider any other policy $\pi'$. If $\pi'$ selects an action $a'$ such that $Q^*(s,a') < Q^*(s,a_E)$ for some $s$, then $\pi'$ is not greedy and thus suboptimal. Specifically, the value difference at that step is:
\begin{align}
&Q^*(s,a_E) - Q^*(s,a') = \big[r^*(s,a_E) - r^*(s,a')\big] \notag \\
&\quad + \gamma \big[\mathbb{E}_{s'}[V^*(s' \mid s,a_E)] - \mathbb{E}_{s'}[V^*(s' \mid s,a')]\big].
\end{align}
This difference is strictly positive when $Q^*(s,a_E) > Q^*(s,a')$, implying $\pi_E$ yields strictly greater expected return. Therefore, $\pi_E$ is optimal under $r^*$.
\end{proof}

\section{Experiments}
We evaluate our proposed method, BiCQL-ML,on a diverse collection of robotics training tasks in MuJoCo simulator, as well as datasets in D4RL benchmark~\cite{d4rl}. Our goal is to demonstrate that our method can effectively learn reward functions from offline expert demonstrations, and that policies trained with the learned rewards outperform prior imitation learning methods.

\subsection{Experimental Setup}
We conduct experiments on four standard MuJoCo continuous control environments: \texttt{HalfCheetah}, \texttt{Hopper}, \texttt{Walker2d}, and \texttt{Ant}, following the setup and expert demonstrations from GAIL~\cite{Ho2016}. To evaluate the robustness of our approach across varying data regimes, we consider two settings: a low-data regime using only 1 expert trajectory, and a medium-data regime with 10 expert trajectories.

All algorithms are trained under a unified experimental protocol to ensure consistent comparison. Each experiment is conducted with 42 random seeds, and we report the mean and standard deviation of the episodic returns over 10 evaluation episodes, measured every 1000 environment steps. Training and evaluation are performed on a machine equipped with an NVIDIA GeForce RTX 4090 GPU.

\subsection{Evaluation Metrics}
Since our algorithm learns only a reward function from offline expert demonstrations, evaluating its effectiveness requires deriving a policy that optimizes the learned reward. To this end, we train a policy using Soft Actor-Critic (SAC) on top of the fixed reward function. In order to compare with other inverse reinforcement learning approaches under a unified evaluation framework, we deploy the resulting policy in the MuJoCo environment and measure standard metrics such as:

\begin{itemize}
    \item \textbf{Average Return:} Mean episode return of the learned policy, averaged over 100 evaluation episodes at regular intervals.
    \item \textbf{Convergence Speed:} Number of environment steps required to reach 90\% of expert performance.
    \item \textbf{Sample Efficiency:} Total number of environment steps used during both reward learning and policy training.
\end{itemize}
This evaluation protocol serves to validate how well the learned reward captures the expert’s intent, by assessing the quality of policies optimized under that reward in a standard continuous control setting.

\subsection{Comparison with Baselines}

We compare our proposed method, BiCQL-ML, with several established imitation learning and inverse reinforcement learning baselines, including Behavioral Cloning (BC)~\cite{Torabi2018}, Discriminator-Actor-Critic (DAC)~\cite{Dac}, and ValueDICE~\cite{ValueDice}. To investigate the robustness and generalization ability of each method under varying levels of supervision, we conduct evaluations under two distinct data regimes: low-data regime with only 1 expert trajectory and high-data regime with 10 expert trajectories. This separation allows us to assess how well each approach performs when expert demonstrations are scarce versus when more guidance is available.

\subsubsection{Performance under Low-Data Regime} As illustrated in Figure 3. our proposed method, BiCQL-ML, outperforms existing baselines in the low-data regime across three out of four MuJoCo environments, demonstrating both faster convergence and higher final performance. This setting, where only a single expert trajectory is available, poses significant challenges for generalization and reward inference.

Behavioral Cloning (BC) fails to learn meaningful policies in all environments, highlighting its poor generalization ability beyond expert distributions when supervision is minimal. Discriminator-Actor-Critic (DAC) exhibits highly unstable learning in Ant, Hopper, and Walker, with occasional spikes in performance but overall inconsistency. It performs relatively better in HalfCheetah, though still inferior to methods based on reward learning. ValueDICE shows more stable behavior compared to DAC and BC, and achieves performance close to BiCQL-ML in several cases. Notably, in the HalfCheetah environment, ValueDICE slightly outperforms BiCQL-ML, suggesting that distribution-matching alone can be competitive in less complex dynamics. However, in more challenging environments such as Ant and Walker, BiCQL-ML yields significantly better performance, both in terms of return and learning stability.

BiCQL-ML enables more stable and data-efficient learning, particularly under limited supervision, by avoiding the pitfalls of adversarial instability and overfitting to narrow expert distributions.

\subsubsection{Performance under High-Data Regime} As shown in Figure 3, all methods benefit from the increased supervision provided by 10 expert trajectories, showing notable improvements over their performance in the low-data regime. This validates the general trend that richer expert data leads to better learning of reward functions and policies across all algorithms.

Despite these improvements, DAC and BC continue to exhibit significantly lower returns compared to BiCQL-ML and ValueDICE. The adversarial nature of DAC and the supervised nature of BC still suffer from instability and poor generalization, especially in more complex environments like Ant and Walker. ValueDICE narrows the performance gap with BiCQL-ML in several environments, particularly in HalfCheetah and Hopper, where both methods achieve near-optimal returns. However, BiCQL-ML continues to exhibit either higher or comparable final performance across all environments, and importantly, demonstrates more stable and faster convergence—especially in the Walker and Ant tasks. Moreover, BiCQL-ML consistently achieves reliable returns even during early stages of training, indicating enhanced sample efficiency and training stability. 

These results indicate that BiCQL-ML scales effectively with data availability. Its sample-efficient reward learning and regularized Q-function optimization allow it to exploit richer expert supervision without overfitting. The consistent performance across different data regimes suggests strong generalization and robustness properties.

\textbf{Overall}, BiCQL-ML outperforms or matches state-of-the-art IRL baselines in both data-scarce and data-rich settings, making it a reliable choice for real-world offline imitation learning problems.

\begin{figure}
    \centering
    \includegraphics[scale=0.45]{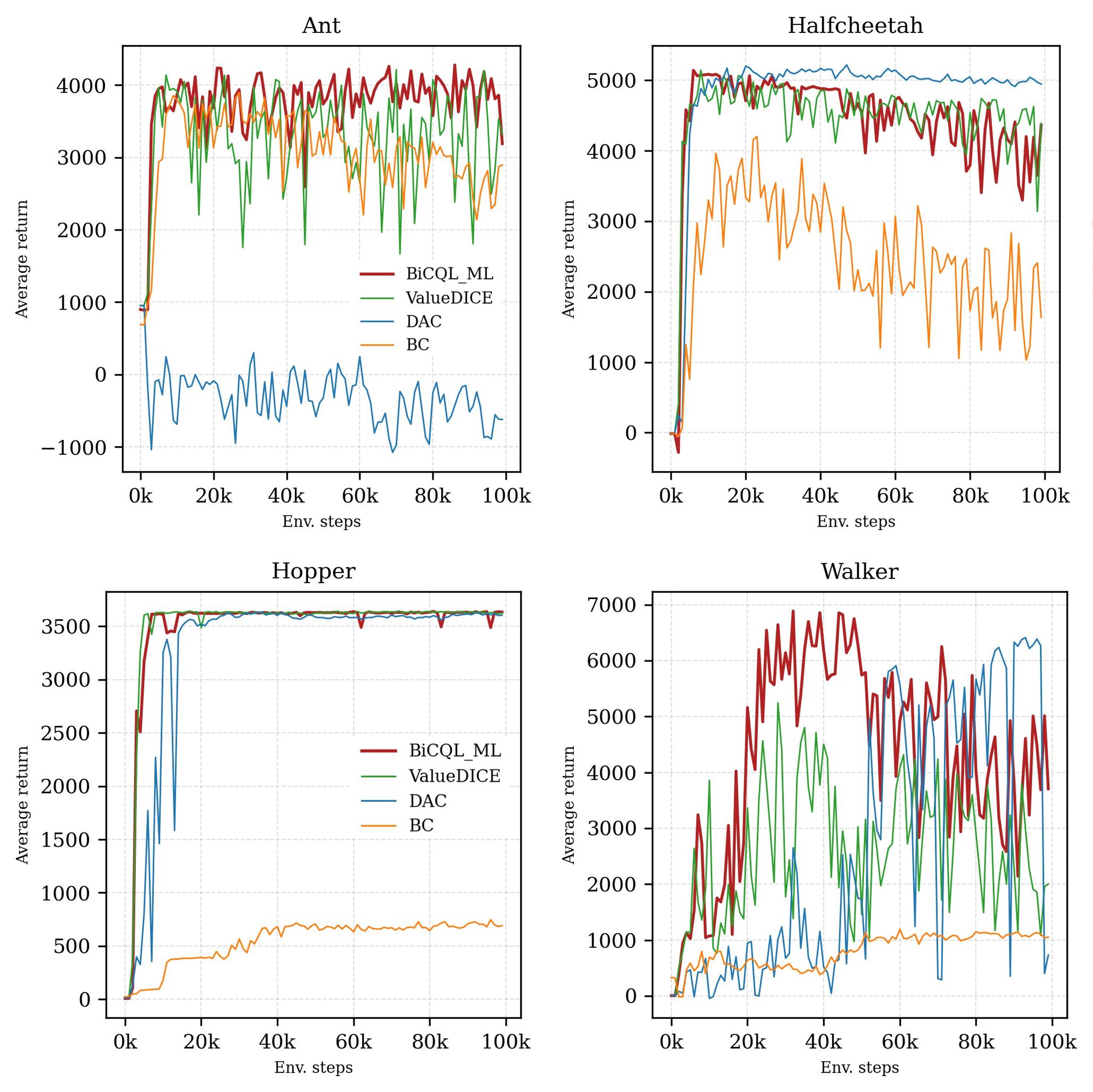}
    \caption{The performance comparison under the low-data regime (single expert demonstration). BiCQL-ML significantly outperforms other baselines across Ant, Halfcheetah, Hopper, and Walker tasks in terms of sample efficiency and final average return.}
\end{figure}

\begin{figure}
    \centering
    \includegraphics[scale=0.45]{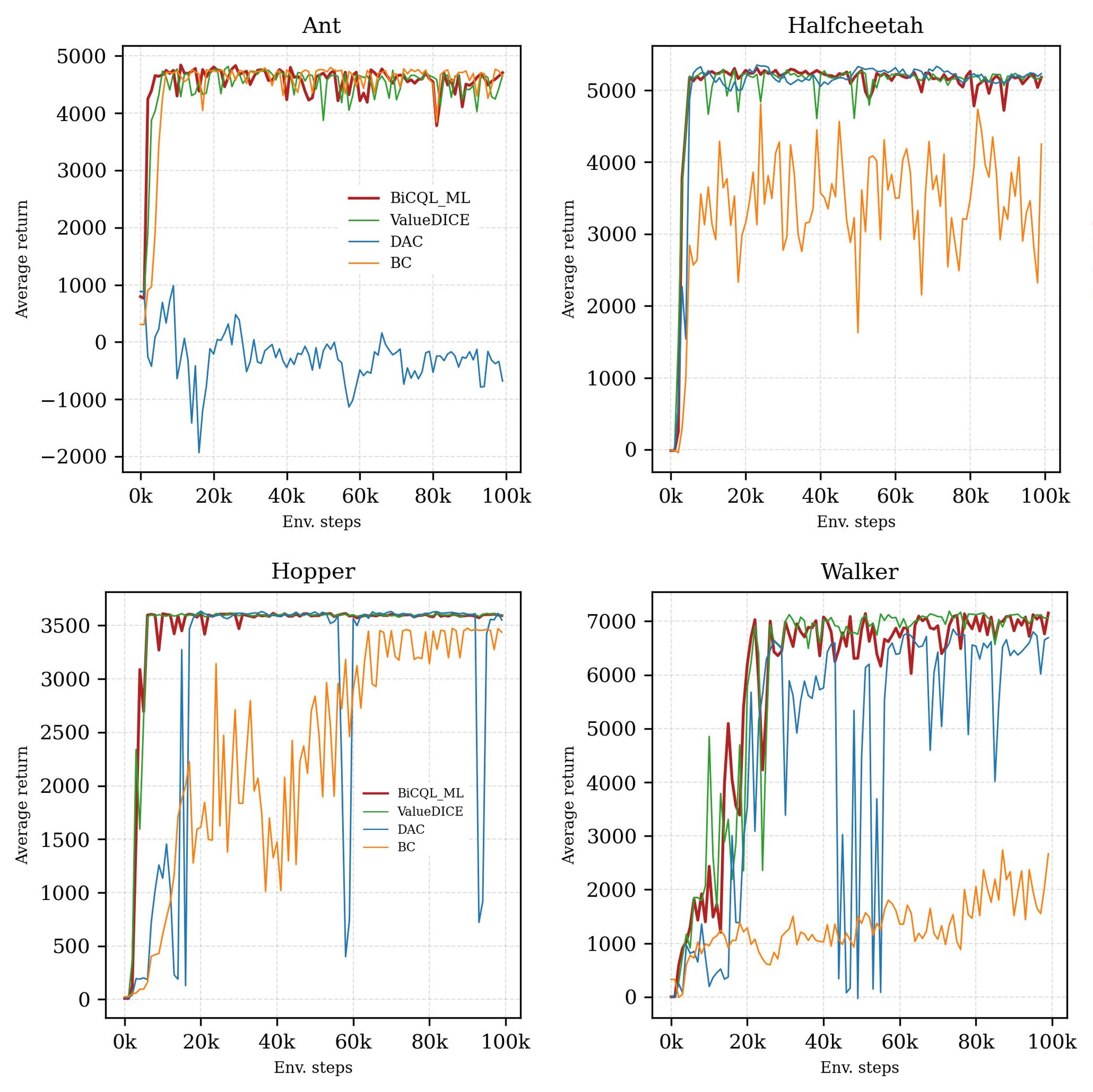}
    \caption{The performance comparison under the high-data regime (10 expert demonstrations). BiCQL-ML achieves consistently high returns and maintains competitive or superior performance to ValueDICE and DAC across all tasks.}
\end{figure}

\section{Conclusion and Future Work}

In this paper, we presented \textit{BiCQL-ML}, a novel offline inverse reinforcement learning (IRL) algorithm that employs a bi-level optimization framework combining Conservative Q-Learning (CQL) and maximum likelihood reward inference. Unlike conventional approaches that rely heavily on explicit policy optimization or adversarial training, our method bypasses these requirements by alternately optimizing a conservative Q-function and a reward function parameterized by maximum likelihood estimation.

Empirical evaluations across diverse MuJoCo tasks in both low-data and high-data regimes demonstrate the robustness and effectiveness of \textit{BiCQL-ML}. Particularly in scenarios with limited expert data, \textit{BiCQL-ML} consistently exhibits superior performance compared to baseline methods such as Behavioral Cloning, DAC, and ValueDICE, achieving higher returns and faster convergence. In data-rich conditions, our method maintains competitive performance, demonstrating improved stability and early-stage reward learning reliability, further underscoring its effectiveness.

For future work, several promising avenues exist. First, integrating adaptive mechanisms for dynamically adjusting conservatism could enhance performance and robustness in varying offline data conditions. Second, exploring theoretical extensions that incorporate additional modalities such as human preference feedback or uncertainty estimates could broaden the applicability of our algorithm in more complex real-world scenarios. Lastly, scaling the framework to more diverse and higher-dimensional tasks beyond MuJoCo environments could further validate and extend the practical utility of our proposed methodology.


\section*{APPENDIX}
\section{Theoretical Analysis: Detailed Proofs}
\label{appendix:theory}

\subsection{Proof of Lemma 1}
\begin{proof}
Let $\theta_k$ and $\theta_k'$ be two different reward parameter iterates at some iteration $k$. After one full alternating update (both lower-level and upper-level), we have:
\[
\theta_{k+1} = G(\theta_k), \quad \theta_{k+1}' = G(\theta_k').
\]
By definition,
\[
\|\theta_{k+1} - \theta_{k+1}'\| = \|F_{\text{upper}}(Q^*(\theta_k)) - F_{\text{upper}}(Q^*(\theta_k'))\|.
\]
By the $L_{\text{ML}}$-Lipschitz continuity of $F_{\text{upper}}$, we have:
\[
\|\theta_{k+1} - \theta_{k+1}'\| \leq L_{\text{ML}} \| Q^*(\theta_k) - Q^*(\theta_k') \|_\infty.
\]
Then, applying the $L_Q$-Lipschitz continuity of $Q^*(\theta)$ from Assumption~\ref{assump:lipschitz} yields:
\[
\|\theta_{k+1} - \theta_{k+1}'\| \leq L_{\text{ML}} L_Q \| \theta_k - \theta_k' \|.
\]
Thus, the update contracts by a factor of at most $L_{\text{ML}} L_Q < 1$ each iteration. By induction, after $n$ steps:
\[
\|\theta_n - \theta_n'\| \leq (L_{\text{ML}} L_Q)^n \|\theta_0 - \theta_0'\|,
\]
which tends to zero as $n \to \infty$. This proves that $G$ is a contraction with ratio $q = L_{\text{ML}} L_Q < 1$, and that the sequence converges to a unique fixed point regardless of initialization.
\end{proof}

\subsection{Proof of Lemma~\ref{lemma:likelihood_opt}}
\begin{proof}
At the optimum $\theta^*$, the gradient of the log-likelihood $\mathcal{L}(\theta)$ vanishes:
\[
    \frac{\partial}{\partial Q_\theta(s,b)} \log \pi_\theta(a_E \mid s)
    =
    \begin{cases}
        1 - \pi_\theta(a_E \mid s), & \text{if } b = a_E, \\[3pt]
        -\,\pi_\theta(b \mid s), & \text{if } b \neq a_E.
    \end{cases}
\]
The stationary condition $\nabla_{Q(s,\cdot)} \mathcal{L}(\theta^*) = 0$ implies:
\[
    \pi_{\theta^*}(a_E \mid s) = 1, \quad \pi_{\theta^*}(b \mid s) = 0 \quad \forall b \neq a_E.
\]
Since $\pi_{\theta^*}(a \mid s) \propto \exp(Q^*(s,a))$, this means:
\[
    Q^*(s,a_E) > Q^*(s,a), \quad \forall a \neq a_E~\text{with}~\pi_E(a \mid s) = 0,
\]
\[
    Q^*(s,a_E) = Q^*(s,a), \quad \forall a \in \mathcal{A}~\text{such that}~\pi_E(a \mid s) > 0.
\]
Hence, the expert’s chosen action $a_E$ maximizes $Q^*(s,a)$ at each state $s$.
\end{proof}

\subsection{Proof of Theorem 1}
\begin{proof}
\textbf{Existence and Uniqueness.}  
By Lemma~\ref{lemma:contraction}, the composite update $G(\theta) := F_{\text{upper}}(F_{\text{lower}}(\theta))$ is a contraction mapping over the complete metric space of parameters $\Theta$. Therefore, by the Banach Fixed-Point Theorem, there exists a unique fixed point $\theta^*$ such that $G(\theta^*) = \theta^*$. Let $Q^* = F_{\text{lower}}(\theta^*)$ be the corresponding conservative $Q$-function. Then, by construction,
\[
F_{\text{upper}}(Q^*) = \theta^*, \quad F_{\text{lower}}(\theta^*) = Q^*.
\]
This implies that
\[
\theta^* = \arg\max_\theta \mathcal{L}(\theta; Q^*), \quad Q^* = \arg\min_Q \mathcal{L_Q}(Q; r_{\theta^*}),
\]
where $\mathcal{L_Q}$ denotes the CQL loss. Hence, $(\theta^*, Q^*)$ satisfies the optimality conditions of fixed points at the bi-level.

\textbf{Convergence of Algorithm.}  
The algorithm alternates between computing:
\[
Q_{k+1} = F_{\text{lower}}(\theta_k), \quad \theta_{k+1} = F_{\text{upper}}(Q_{k+1}).
\]
This can be written as the iteration $\theta_{k+1} = G(\theta_k)$. Since $G$ is a contraction, the sequence $\{\theta_k\}$ converges to $\theta^*$. By continuity of $F_{\text{lower}}$, it follows that $Q_k := F_{\text{lower}}(\theta_k)$ converges to $Q^*$. Therefore, the entire pair $(\theta_k, Q_k)$ converges to $(\theta^*, Q^*)$, completing the proof.
\end{proof}

\subsection{Algorithm Details}

\begin{center}
\begin{minipage}{1.00\textwidth}
\begin{algorithm}[H]
  \caption{BiCQL-ML}
  \label{alg:bi_level_irl}
  \begin{algorithmic}[1]
    \Require Offline dataset $\mathcal{D} = {(s, a, s')}$, subset of expert demonstrations $\mathcal{D}E \subseteq \mathcal{D}$, discount factor $\gamma$, conservatism coefficient $\alpha$, learning rates $\eta_Q, \eta_R$, update frequencies $K_Q, K_R$
    \State Initialize reward network parameters \(\theta\) via Xavier initialization
    \State Initialize  Q-function $Q_\phi$; set target Q-network parameters $\phi^- \leftarrow \phi$
    \While{not converged}
      \Statex \hfill \textit{// Lower-Level: Conservative Q-function update}
      \For{$k = 1, \dots, K_Q$}
        \State Sample mini-batch $\{(s_i,a_i,s'_i)\}_{i=1}^B \sim \mathcal{D}$
        \State Compute soft Bellman targets: $y_i \leftarrow R_\theta(s_i,a_i) + \gamma \log\sum_{a'}\exp(Q_{\phi^-}(s_i', a'))$
        \State Compute Bellman error: $\mathcal{L}_{\mathrm{BE}} \leftarrow \frac{1}{B} \sum_i \left(Q_\phi(s_i,a_i) - y_i\right)^2$
        \State Compute CQL regularizer: 
        \[
          \mathcal{L}_{\mathrm{CQL}} \leftarrow \frac{1}{B} \sum_i \left[\log \sum_a \exp(Q_\phi(s_i,a)) - Q_\phi(s_i,a_i)\right]
        \]
        \State Update Q-function: 
        \[
          \phi \leftarrow \phi - \eta_Q \nabla_\phi \left( \mathcal{L}_{\mathrm{BE}} + \alpha \mathcal{L}_{\mathrm{CQL}} \right)
        \]
      \EndFor
      \State Periodically update target network: $\phi^- \leftarrow \phi$

      \Statex \hfill \textit{// Upper-Level: Reward function update}
      \For{$k = 1, \dots, K_R$}
        \State Sample mini-batch $\{(s_i,a_i,s'_i)\}_{i=1}^B \sim \mathcal{D}_E$
        \State Compute soft advantage regression loss:
        \[
            \mathcal{L}_r(\theta) \leftarrow \frac{1}{B} \sum_{i=1}^B \left( r_\theta(s_i,a_i) - \left( Q_\phi(s_i,a_i) - \gamma \log \sum_{a'} \exp(Q_\phi(s'_i, a')) \right) \right)^2
        \]
        \State Update reward parameters: 
        \[
          \theta \leftarrow \theta - \eta_r \nabla_\theta \mathcal{L}_r(\theta)
        \]
      \EndFor
    \EndWhile
    \State \Return $R_\theta$ and $Q_\phi$
  \end{algorithmic}
\end{algorithm}
\end{minipage}
\end{center}

\end{document}